\documentclass[letterpaper]{article}

\usepackage[preprint]{aaai2027}

\usepackage{amsmath,amssymb,amsthm,mathtools,bm}

\usepackage{needspace}
\usepackage{graphicx}
\usepackage{subcaption}
\usepackage{longtable}
\usepackage{float}
\usepackage{booktabs}
\usepackage{array}
\usepackage{multirow}

\usepackage{enumitem}
\usepackage{natbib}
\usepackage{url}

\graphicspath{{figures/}}

\setlist[itemize]{
    leftmargin=1.05em,
    itemsep=0.18em,
    topsep=0.25em
}

\setlist[enumerate]{
    leftmargin=1.20em,
    itemsep=0.18em,
    topsep=0.25em
}

\newtheorem{definition}{Definition}
\newtheorem{proposition}{Proposition}
\newtheorem{theorem}{Theorem}

\newcommand{\R}{\mathbb{R}}
\newcommand{\E}{\mathbb{E}}
\newcommand{\X}{\mathcal{X}}
\newcommand{\Y}{\mathcal{Y}}

\newcommand{\Acal}{\mathcal{A}}
\newcommand{\Fcal}{\mathcal{F}}

\newcommand{\poly}{\operatorname{poly}}
\newcommand{\argmin}{\operatorname*{arg\,min}}

\newcommand{\Tr}{\operatorname{tr}}

\newcolumntype{L}[1]{>{\raggedright\arraybackslash}p{#1}}

\title{The Structure Gap: When More Data Cannot Replace\\
Structured Reasoning}

\author{
    Kheer Sagar Patel, Mojtaba Soltanalian
}

\affiliations{
    University of Illinois Chicago
}

\begin{document}

\maketitle

\begin{abstract}
A persistent belief in modern AI is that sufficiently large learners trained on sufficiently much data can make explicit domain knowledge optional. This paper examines a limitation of that belief. We define \emph{structure-encoded problems}: tasks whose input-output map is determined, or sharply constrained, by known structure such as a solver, symmetry group, graph, physical law, measurement operator, grammar, combinatorial search space, or optimization objective. For such tasks, the central question is not expressivity---a dense neural network may contain the desired map---but finite-data identification, extrapolation, and certification. We introduce the \emph{structure gap}: the finite-resource gap between a method that receives the governing structure and a learner that receives only examples. The strongest useful form of the claim is asymptotic rather than absolute: structure-agnostic learning may approach a structured solver as coverage grows, but every finite dataset can leave residual uncertainty that the structure resolves by construction. We prove this exactly for a Gaussian inverse problem: the known-operator Wiener solver attains Bayes risk, while the dense empirical learner in the correct linear class has persistent expected excess risk. We also prove a cell-coverage lower bound and an unobserved-orbit non-identifiability result. Controlled experiments in known-operator inversion, sparse recovery, and graph reachability with size extrapolation instantiate the mechanisms. When validated structure is available, it can constrain the hypothesis space before learning begins, leaving learning to address components the structure does not specify.
\end{abstract}

\section{Introduction}

Deep learning has changed what is possible in perception, inference, language, and control. Its success has also encouraged a powerful informal claim: with enough data and capacity, explicit structure and domain knowledge become optional. The claim is sometimes right. When the relevant structure is unknown, approximate, or too complex to specify, learned representations may be the best available tool. But in many scientific, engineering, and symbolic domains the relevant structure is already known. A parity-check matrix is not a statistical hint; it is the code. A graph relabeling symmetry is not an empirical regularity; it is a mathematical equivalence. A Fourier basis, array manifold, grammar, shortest-path recurrence, measurement matrix, likelihood, conservation law, or convex objective is information about the task.

This paper studies what happens when that information is withheld from the learner. Broad feedforward networks can approximate large function classes on compact domains \citep{Cybenko1989,Hornik1989}, and depth can improve approximation efficiency \citep{EldanShamir2016,Telgarsky2016,Yarotsky2017}. Universal approximation, however, does not imply \emph{efficient learnability}. A fully connected network (FCN) may contain, somewhere in its parameter space, a representation of a structured solver. That does not mean finite examples identify the representation, certify behavior on unobserved regimes, or recover the law that generated the labels.

\begin{figure*}[t]
    \centering
    \includegraphics[width=\textwidth]{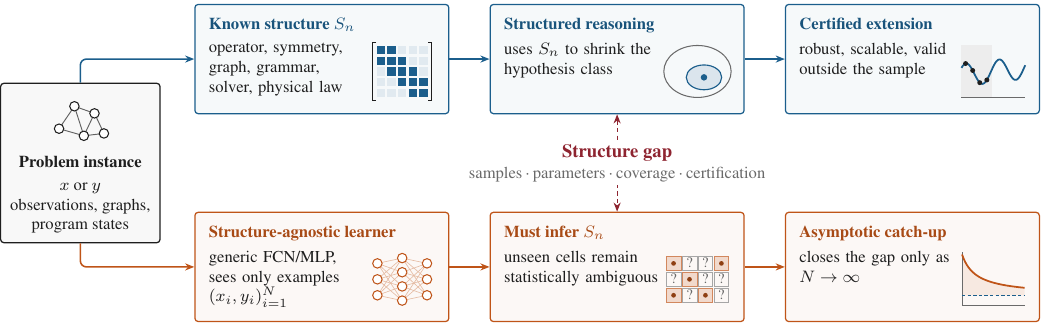}
    \caption{Conceptual view of the structure gap. A structure-aware system uses a known law, solver, graph, symmetry, grammar, or objective to restrict the hypothesis class before learning. A structure-agnostic learner can fit examples, but must infer the law from data and therefore pays with samples, parameters, coverage, and weaker extrapolation. The gap is a finite-resource phenomenon: the data-only path may approach the structured path asymptotically, but the structure certifies the extension at finite sample size.}
    \label{fig:concept}
\end{figure*}

We call the resulting resource difference the \emph{structure gap}. The gap can be a sample-complexity gap, as in convolutional or equivariant architectures versus fully connected networks; a certification gap, because finite examples cannot rule out alternative extensions that agree on the data but disagree on unobserved orbits; a combinatorial coverage gap, because a learner must see many supports, assignments, graph configurations, or symbolic cases; or an asymptotic-only gap, because the best structure-free learner may approach the structured solver as the data size \(N\to\infty\), but no finite dataset fully replaces the governing law.

\paragraph{Thesis.}
Known structure is information. If a learner ignores it, it must purchase that information back using data, parameters, computation, precision, coverage, and optimization luck. In some problem classes the price is exponential or infinite; in others it is finite but large enough that the solver is the correct baseline.

\paragraph{Contributions.}
We make four contributions. First, we define structure-encoded problems and structural separation as finite-resource notions rather than rhetorical claims against neural networks. Second, we give lower-bound mechanisms that formalize why finite data cannot certify global structured laws without the structure. Third, we prove an exact asymptotic-catch-up theorem in the most favorable possible case for the dense learner: the Bayes map is linear and lies in the learner class. Fourth, we provide reproducible numerical studies, including a graph-reachability experiment that exhibits a graded hierarchy: a dense FCN interpolates the training graph size, a message-passing graph neural network (GNN) retains much stronger size extrapolation, and exact breadth-first search (BFS) closes the remaining gap.

\paragraph{Scope of the claim.}
The paper is class-relative. We compare methods that encode the correct structure with learners that are deliberately denied it. A convolution, GNN, unrolled optimizer, differentiable parser, physics-informed model, or tool-using system is therefore structure-aware in our taxonomy, even if it is neural. We also do not claim that wrong or approximate structure is always beneficial. The central implication is that, when the correct structure is known in a compact form, relearning its consequences from examples imposes an avoidable finite-data burden.

\section{Related Work}

\paragraph{Approximation, learnability, and inductive bias.}
Universal approximation theorems establish broad representational capacity \citep{Cybenko1989,Hornik1989}; approximation and depth-separation results refine the efficiency picture \citep{Barron1993,EldanShamir2016,Telgarsky2016,Yarotsky2017}. Statistical learning theory and no-free-lunch results emphasize that generalization requires assumptions on the problem class \citep{WolpertMacready1997}. Recent work revisits inductive bias using Kolmogorov complexity and the no-free-lunch perspective \citep{Goldblum2024}. Our contribution is to treat correct structure not merely as a useful bias, but as side information that changes the identification problem before data arrive.

\paragraph{Equivariance and architectural structure.}
Convolution and equivariance are canonical examples of structure encoded in an architecture. Group equivariant networks exploit known symmetries \citep{CohenWelling2016}; Deep Sets encode permutation invariance \citep{Zaheer2017}; graph networks and message passing encode relational locality \citep{Battaglia2018,KipfWelling2017,Gilmer2017,Xu2019}; and geometric deep learning organizes architectures around groups, graphs, manifolds, and gauges \citep{Bronstein2021}. Theory has established sample-complexity and computational separations between convolutional/equivariant models and fully connected networks \citep{Du2018,LiZhangArora2021,MalachShalev2021,Lahoti2024,ElesedyZaidi2021,KianiNguyen2020}. Recent work studies hardness and expressivity under symmetries \citep{Kiani2024,Pacini2024,Pacini2025}. These results support a broader principle: structure changes not only constants but the learning problem itself.

\paragraph{Model-based, algorithmic, and neuro-symbolic AI.}
Compressed sensing and inverse problems show how known operators and priors make high-dimensional inference tractable \citep{CandesRombergTao2006,Donoho2006,TroppGilbert2007,Wainwright2009}. ISTA/FISTA and unrolled networks explicitly mix solvers with trainable components \citep{Daubechies2004,BeckTeboulle2009,GregorLeCun2010,Monga2021,pinilla2023unfolding,khobahi2020modelbased, khobahi2021lordnet, Shlezinger2025DeepUnfolding}. Neural algorithmic reasoning trains networks to imitate or align with classical algorithms \citep{Velickovic2020,Velickovic2022CLRS,Numeroso2023}, while neuro-symbolic AI argues for combining learned representations with symbolic knowledge and reasoning \citep{GarcezLamb2023,Matalon2023,Colelough2025}. The structure gap supplies a common resource-separation language for these lines of work.

\section{Structure-Encoded Problems}

Let \(f_n:\X_n\to\Y_n\) be a family of target maps indexed by dimension or problem size. Let \(S_n\) denote side information describing the valid structure: a group action, graph, measurement operator, program, grammar, constraint set, likelihood, physical law, objective, or solver.

\begin{definition}[Structure-aware and structure-agnostic methods]
A method class \(\Acal_n\) is \emph{structure-aware} if its elements may access \(S_n\) at design or inference time. A learner class \(\Fcal_{n,P}\) is \emph{structure-agnostic} if its architecture and training rule do not enforce the constraints encoded by \(S_n\); it receives only input-output examples and has \(P\) trainable or effective parameters.
\end{definition}

\begin{definition}[Structure gap and structural separation]
A problem family exhibits a \emph{structure gap} at error level \(\epsilon\), relative to a distributional family and learner class, if a structure-aware method achieves risk at most \(\epsilon\) using resources \(\poly(n,\log(1/\epsilon))\), while a structure-agnostic learner in the considered class requires asymptotically larger resources in at least one axis: samples, parameters, computation, precision, coverage of modes/orbits, or dimension-specific retraining. If the required resource is super-polynomial or exponential, we call the gap a \emph{structural separation}.
\end{definition}

\begin{figure*}[t]
    \centering
    \includegraphics[width=\textwidth]{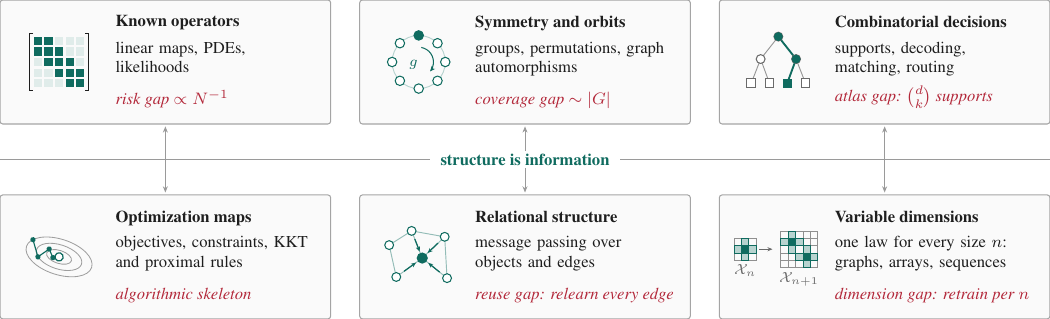}
    \caption{Problem classes where structure changes the learning problem. Each box summarizes a compact structural description---known operators, symmetry, combinatorial rules, optimization constraints, relational structure, or variable dimension---that induces a large family of input-output constraints. When the structure is correct, it acts as information that shrinks the effective hypothesis space; a structure-agnostic learner must infer these implications from examples.}
    \label{fig:taxonomy}
\end{figure*}

Note that resource is central to our formulation. For instance, a structure gap does not imply that an FCN cannot represent the map. It says that representation is not enough: the learner must identify the right extension from finite data. Figure~\ref{fig:taxonomy} summarizes common ways in which this identification problem demands many samples, parameters, or regime-covering examples. The three mechanisms below make this link explicit: operator identification (M1), orbit coverage (M2), and combinatorial atlases (M3).

\paragraph{M1: operator identification.}
If the correct map is determined by a known operator or likelihood, a classical solver can use it immediately. A learner that does not encode it must estimate the operator-induced map from samples. Even if the hypothesis class contains the exact map, finite data leaves parameter uncertainty.

\paragraph{M2: orbit coverage.}
If \(f\) is invariant or equivariant under a group \(G\), data from one point implies behavior on its orbit only if the learner knows the group action. Without encoded equivariance, finite samples cover only fragments of the orbit space.

\paragraph{M3: combinatorial atlases.}
Many structured maps are piecewise algebraic with regions indexed by discrete objects: supports, assignments, shortest paths, codewords, parse states, active constraints, or matchings. Solvers navigate this atlas using symbolic or optimization rules. Structure-agnostic learners must infer many regions and switching boundaries.

\section{Why More Data May Only Approach Structure}

In many finite problems, sufficiently expressive learners are consistent and approach the structured solver as \(N\to\infty\). In the context of neural networks, we do not claim that neural networks cannot learn structure. Our sharper claim is: for some problem classes, the best structure-free learner can only \emph{asymptotically} approach the performance of a solver that has the correct structure, and the finite-data price of approach is controlled by the coverage complexity of the hidden structure.

%\Needspace{5\baselineskip}
\subsection{Known-Operator Solvers and \mbox{Asymptotic Catch-Up}}

Consider a setting in which the learner class already contains the exact solution. The Bayes-optimal map is linear, so a one-layer dense network already contains the exact solution. The only missing ingredient is the known operator.

Consider
\begin{equation}
    x\sim \mathcal{N}(0,I_d),\qquad y = A x + \xi,
    \qquad \xi\sim \mathcal{N}(0,\sigma^2 I_m),
    \label{eq:gaussian-inverse}
\end{equation}
with known \(A\in\R^{m\times d}\) for the classical solver. The risk of a linear estimator \(W\) is \(R(W)=\E \left[\|Wy-x\|_2^2\right]\). By Gaussian conditioning, the Bayes-optimal known-operator solver is the Wiener inverse
\begin{equation}
    W_\star = A^\top (AA^\top + \sigma^2 I_m)^{-1}.
    \label{eq:wiener}
\end{equation}
Let \(\Sigma_e=\operatorname{Cov}(x-W_\star y)\).

\begin{theorem}[Exact asymptotic catch-up]
\label{thm:gaussian}
Let \((x_i,y_i)_{i=1}^N\) be i.i.d. samples from \eqref{eq:gaussian-inverse}. Let
\begin{equation}
    \widehat{W}_N \in \argmin_W \sum_{i=1}^N \|Wy_i-x_i\|_2^2
\end{equation}
be the dense least-squares learner. If \(N>m+1\) and \(AA^\top+\sigma^2I_m\) is nonsingular, then, with expectation over the training sample and an independent test pair,
\begin{equation}
    \E R(\widehat{W}_N)
    = R(W_\star)+\frac{m \, \Tr(\Sigma_e)}{N-m-1}.
    \label{eq:exact-gap}
\end{equation}
Consequently, for every finite \(N>m+1\) with \(m>0\) and \(\Tr(\Sigma_e)>0\), the expected risk of the dense empirical learner is strictly larger than that of the known-operator solver, and it approaches the solver only as \(N\to\infty\).
\end{theorem}

\noindent\emph{Proof sketch.} Write \(x=W_\star y+e\), where \(e\) is independent of \(y\) and has covariance \(\Sigma_e\). Stacking samples gives \(\widehat{W}_N^\top-W_\star^\top=(Y^\top Y)^{-1}Y^\top E\). Conditional on \(Y\), the excess risk is \(\Tr(\Sigma_e)\Tr\{\Sigma_y(Y^\top Y)^{-1}\}\), with \(\Sigma_y=AA^\top+\sigma^2I_m\). Since \(Y^\top Y\) is Wishart and \(\E[(Y^\top Y)^{-1}]=\Sigma_y^{-1}/(N-m-1)\) \citep{Anderson2003}, substitution gives \eqref{eq:exact-gap}. A full derivation is in the supplementary material. \(\square\)

The result holds in a linear model where the learner class already contains the exact solver: estimating an operator-induced map from examples introduces positive finite-sample excess risk. More data closes the gap; it does not make the structure irrelevant at finite \(N\). A one-layer FCN is exactly a dense linear map, and a deeper FCN contains this map as a special case; hence the result applies directly to the most favorable dense-neural baseline for this problem. Modern overparameterization can change optimization and implicit regularization, but it cannot erase the statistical fact that a learner denied \(A\) must identify an operator-induced rule from finite samples. If a neural method encodes \(A\), the likelihood, or the solver update, it has moved to the structure-aware side of the comparison.

\subsection{A Cell-Coverage Lower Bound}

Many symbolic and structured problems are governed by regimes or cells: supports in sparse recovery, active sets in optimization, parse states in grammars, codewords in decoding, graph configurations in algorithms, or group orbits in equivariance. The following lower bound formalizes the cost of unseen cells.

\begin{proposition}[Finite cell coverage]
Let \(\X\) be a union of disjoint cells \(C_1,\ldots,C_M\) with probabilities \(\left\{p_j\right\}\). Consider the full class of binary cellwise-constant labelings \(h_b(x)=b_j\) for \(x\in C_j\), indexed by \(b\in\{0,1\}^M\). A structure-aware solver receives \(b\); an example-only learner receives \(N\) i.i.d. labeled samples. Under 0-1 loss, the minimax expected risk of any example-only learner is at least
\begin{equation}
    \frac{1}{2}\sum_{j=1}^M p_j(1-p_j)^N .
    \label{eq:cell-bound}
\end{equation}
For uniform cells, this is \(\frac{1}{2}(1-1/M)^N\), whereas the structure-aware solver has zero risk.
\end{proposition}

\noindent\emph{Proof sketch.} Draw the cell labels \(b_j\) independently and uniformly from \(\{0,1\}\). If cell \(C_j\) is absent from the training set, then \(b_j\) remains a fair coin conditional on all observed data, so every predictor incurs conditional error \(1/2\) on that cell. The absence probability is \((1-p_j)^N\); summing over test-cell mass and applying Yao's minimax principle \citep{Yao1977} yields \eqref{eq:cell-bound}. If \(M\) is exponential, the coverage burden is exponential. The supplementary material gives the full proof, a continuous-covering analogue, and a sparse-atlas interpretation. \(\square\)

\subsection{Finite Samples Cannot Certify Unobserved Orbits}

Let a finite group \(G\) act on \(\X\), and let \(\rho\) be a representation on \(\Y\). A map \(f:\X\to\Y\) is equivariant if \(f(gx)=\rho(g)f(x)\).

\begin{proposition}[Unobserved-orbit non-identifiability]
Let \(T\subset\X\) be finite and suppose an orbit \(O=\{gx:g\in G\}\) has \(O\cap T=\emptyset\). If the output space contains two admissible values at a representative of \(O\) compatible with its stabilizer (as in scalar invariant labels or free actions), then there exist two \(G\)-equivariant maps \(f_1,f_2:\X\to\Y\) that agree on \(T\) but disagree on \(O\).
\end{proposition}

\noindent\emph{Proof sketch.} Define the two maps identically on orbits intersecting \(T\). On the unobserved orbit choose two distinct admissible values at a representative and extend equivariantly over the orbit. A more detailed derivation is in the supplementary material. \(\square\)

This is why augmentation is not the same as equivariance. Augmentation samples orbit elements; an equivariant architecture or symbolic rule constrains the entire orbit by construction.

\begin{figure}[t]
    \centering
    \includegraphics[width=\linewidth]{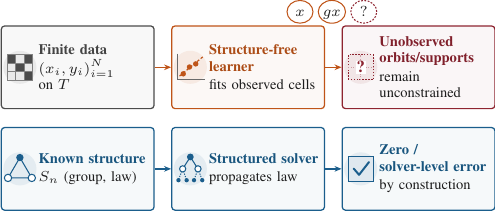}
    \caption{Finite data fits observed examples, but does not certify unobserved orbits, supports, or symbolic regimes. A structure-free learner remains underconstrained outside the sample, whereas a structured solver propagates the governing law by construction.}
    \label{fig:orbit-theory}
\end{figure}

\section{Problem Classes}

Representative classes include known inverse problems, symmetry and invariance, sparse and low-dimensional signals, spectral or convolutional signal models, optimization-defined maps, symbolic and algorithmic reasoning, and physical or causal systems. The common pattern is that the structured solver receives a short description of a large set of implications: all shifts, all relabelings, all supports, all active constraints, all parse states, or all graph sizes. A data-only learner must infer those implications from observed examples. This is why the relevant distinction is not old AI versus new AI, or symbolic versus neural; it is whether the system makes judicious use of the known law as side information. Supplementary material gives a table of examples and solvers.

\section{Numerical Investigations}

The experiments isolate the effect of withholding structure. The baselines are fully connected networks or dense empirical learners, not because these are the strongest possible neural methods, but because they represent the structure-agnostic position. A GNN, equivariant network, unrolled optimizer, differentiable parser, or physics-informed model belongs on the structure-aware side of the thesis. If such a model closes the gap, it supports rather than refutes the claim: it has encoded part of the missing structure. Thus the experiments are not offered as universal benchmark dominance; they are controlled probes of the mechanisms in the theory. Experiment configurations, numerical results, and figure-generation scripts are included in the accompanying code-and-data archive. Additional implementation details appear in the supplementary material.

\begin{figure*}[t]
    \centering
    \includegraphics[width=\textwidth]{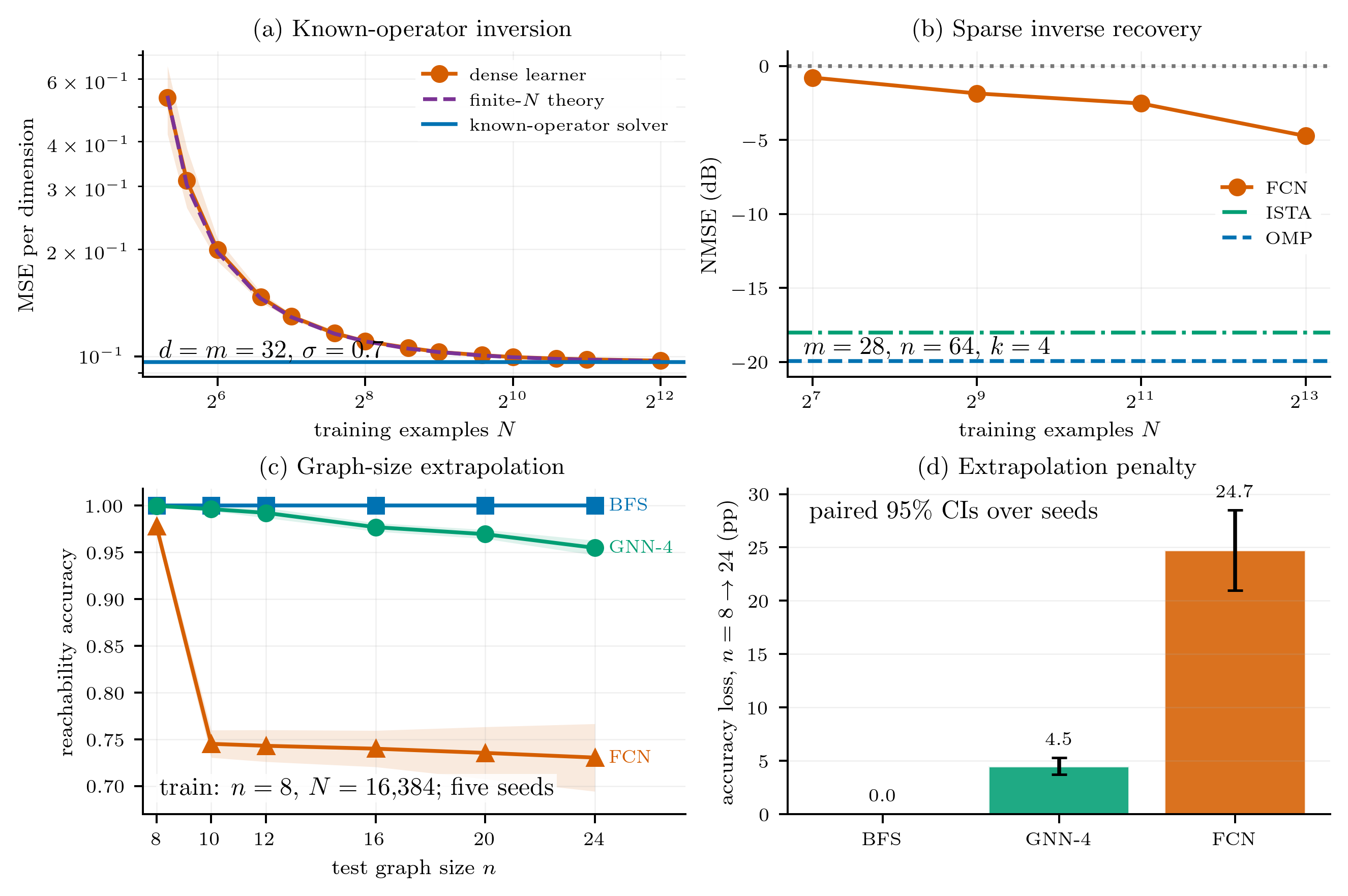}
    \caption{Structured information improves finite-data performance and extrapolation. (a) In known-operator inversion, the dense learner contains the Bayes map but must estimate it from examples; its finite-sample risk follows Theorem~\ref{thm:gaussian}. (b) In sparse recovery, OMP/ISTA use the sensing matrix and sparsity model directly, while the FCN must learn a support-indexed inverse. (c) All learned graph models are trained only on balanced 8-node reachability instances with \(N=16{,}384\). The FCN nearly solves the training graph size but degrades on larger graphs; a four-layer message-passing GNN retains much stronger performance, while BFS is exact. Bands are 95\% confidence intervals over five seeds. (d) Paired accuracy loss from 8 to 24 nodes. Tests are within a shared 24-node padded representation envelope.}
    \label{fig:numerics}
\end{figure*}

\paragraph{Known-operator inversion.}
We instantiate \eqref{eq:gaussian-inverse} with \(d=m=32\) and \(\sigma=0.70\). Figure~\ref{fig:numerics}(a) compares the dense least-squares learner to the known-operator Wiener solver. The empirical curve matches Theorem~\ref{thm:gaussian}: the learner approaches the solver, but the excess is positive for every finite \(N\) tested.

\paragraph{Sparse inverse recovery.}
We generate \(k\)-sparse signals \(x\in\R^{64}\) with \(k=4\), measurements \(y=Ax+\varepsilon\) with \(A\in\R^{28\times64}\), and SNR 25 dB. OMP and ISTA receive \(A\) and the sparsity model; the FCN sees only examples \((y,x)\). The support atlas has \(\binom{64}{4}\) possible supports. Figure~\ref{fig:numerics}(b) shows that the FCN improves with samples but remains far above the solvers' reconstruction error at the considered finite sample sizes.

\paragraph{Graph reachability and size extrapolation.}
We train on balanced source-target reachability problems over 8-node Erd\H{o}s--R\'{e}nyi graphs and test on graphs with 8--24 nodes inside a shared 24-node padded representation. The comparison separates information patterns: the FCN receives a flattened adjacency matrix and active-node mask, the message-passing GNN encodes relational locality and weight sharing \citep{Gilmer2017}, and BFS executes the exact graph algorithm \citep{Cormen2022}. With \(N=16{,}384\) training graphs, the FCN attains \(97.79\%\) accuracy at 8 nodes but \(73.05\%\) at 24 nodes; GNN-4 attains \(99.97\%\) and \(95.48\%\), respectively; BFS is exact. Across five seeds, the paired 8-to-24-node loss is \(24.74\pm3.76\) percentage points for the FCN and \(4.49\pm0.79\) for GNN-4, with 95\% confidence intervals (Figure~\ref{fig:numerics}(c,d)). The result exhibits a graded structure gap: relational inductive bias closes most of the dense learner's extrapolation penalty, and explicit algorithmic structure closes the remainder. This setting follows the neural algorithmic-reasoning motivation of learning graph procedures and testing beyond the training size \citep{Velickovic2020,Velickovic2022CLRS}.

Supplementary material provides matched-filter and permutation-equivariance controls, complete GNN-depth curves, AUROC results, and implementation details.

\section{Implications for AI}

\paragraph{Structure is not merely a regularizer.}
Regularization chooses among hypotheses after seeing data. Correct structure changes the hypothesis space before data arrive. A graph, grammar, solver, or conservation law can imply infinitely many labels or constraints. Treating it as just another empirical bias underestimates its information content.

\paragraph{The right baseline is solver-aware.}
In structured domains, the relevant comparison is not neural versus non-neural. It is structure-aware versus structure-agnostic. A GNN, convolution, differentiable parser, physics-informed model, or unrolled optimizer may be neural, but it is not structure-free. Conversely, a black-box network trained on many examples may be powerful, but it pays to rediscover what the solver already knows.

\paragraph{Large models do not repeal coverage.}
Scaling model size can improve representation and optimization. It does not eliminate the need to identify unobserved cells, supports, or regimes. A larger FCN may interpolate the observed data more flexibly, but without the structural law it cannot certify the correct extension. This is why retrieval, tools, symbolic solvers, program execution, and domain-grounded modules remain central to reliable AI systems.

\paragraph{Use learning where the structure ends.}
The strongest systems will often combine both sides: symbolic and algorithmic components for known laws, neural components for perception, model mismatch, nuisance parameters, and unknown residuals. This is the common thread behind model-based deep learning, neural algorithmic reasoning, and neuro-symbolic AI.

\section{Limitations}

The structure gap is not a no-free-lunch concept. It is a conditional statement: if a compact, correct structural description is available, then denying it to the learner can create a finite-resource identification problem. If the structure is wrong, incomplete, or too costly to specify, structure-aware methods can be biased or brittle; then learning the residual, learning the structure, or using weaker priors may be preferable. The experiments isolate specific mechanisms using fully connected networks as structure-agnostic baselines. Graph extrapolation is evaluated within a shared 24-node padded representation under a controlled Erd\H{o}s--R\'{e}nyi generator. 

Finally, asymptotic consistency does not contradict the paper. In several regimes the dense learner does converge; the point is that convergence may be toward a solver that already had the law at sample size zero. For practice, the message is diagnostic: before scaling data and parameters, ask whether structure exists. If so, the default should be to encode it and learn only what remains unknown.

\section{Conclusion}

We introduced the structure gap: a finite-resource separation between systems that use known structure and learners that see only examples. The message is not that deep learning is weak. The message is that structure is crucial information. When a solver, symmetry, graph, grammar, measurement operator, or physical law is correct, a data-only learner must relearn consequences that the structured method obtains by construction.

We do not claim that neural networks cannot match structured solvers. Often, with infinite data and suitable optimization, they can approach the same risk. What matters in practice is that, for important problem classes, finite datasets may leave unobserved regimes that access to the known structure would otherwise resolve. In AI systems that must be reliable, extrapolative, and certifiable, more data is not a substitute for knowing the problem.

\clearpage
\onecolumn
\appendix
\setcounter{figure}{0}
\renewcommand{\thefigure}{S\arabic{figure}}
\setcounter{table}{0}
\renewcommand{\thetable}{S\arabic{table}}
\begin{center}
{\Large\bfseries Supplementary Material}\\[0.7em]
{\large The Structure Gap: When More Data Cannot Replace Structured Reasoning}\\[0.6em]
Kheer Sagar Patel and Mojtaba Soltanalian
\end{center}
\vspace{1.2em}
\section{Proof of the Gaussian asymptotic-catch-up theorem}
Let \(\Sigma_y=AA^\top+\sigma^2I_m\). For Gaussian variables,
\[
  \E[x\mid y]=W_\star y,\qquad W_\star=A^\top\Sigma_y^{-1},
\]
and the residual \(e=x-W_\star y\) is independent of \(y\) with covariance
\[
  \Sigma_e=I_d-A^\top\Sigma_y^{-1}A.
\]
Stack the training features, targets, and residuals as \(Y\in\R^{N\times m}\), \(X\in\R^{N\times d}\), and \(E\in\R^{N\times d}\). Ordinary least squares gives
\begin{align}
    \widehat W_N^\top
    &= (Y^\top Y)^{-1}Y^\top X \\
    &= W_\star^\top+(Y^\top Y)^{-1}Y^\top E .
\end{align}
For a fresh feature \(y\), the conditional excess risk is
\begin{align}
&\E\left[\|(\widehat W_N-W_\star)y\|_2^2\mid Y\right] \\
&=\E\left[\Tr\{(\widehat W_N-W_\star)\Sigma_y(\widehat W_N-W_\star)^\top\}\mid Y\right] \\
&=\Tr(\Sigma_e)\Tr\{\Sigma_y(Y^\top Y)^{-1}\}.
\end{align}
Since \(Y^\top Y\sim\mathcal W_m(\Sigma_y,N)\) and \(N>m+1\),
\[
    \E[(Y^\top Y)^{-1}]=\frac{\Sigma_y^{-1}}{N-m-1}.
\]
Taking expectation over \(Y\) yields
\[
    \E[R(\widehat W_N)]-R(W_\star)
    =\frac{m\,\Tr(\Sigma_e)}{N-m-1}.
\]
This proves the theorem. The expectation is over both the training sample and an independent test pair. The condition \(\Tr(\Sigma_e)>0\) excludes the noiseless degenerate case in which there is no irreducible conditional uncertainty.

\section{Proof of the finite-cell lower bound}
Let \(\mathcal X=\bigcup_{j=1}^M C_j\), with \(\Pr(X\in C_j)=p_j\). Consider the full class of binary cellwise-constant labelings \(h_b(x)=b_j\) for \(x\in C_j\), indexed by \(b\in\{0,1\}^M\). Place the product prior \(b_j\stackrel{\mathrm{i.i.d.}}{\sim}\mathrm{Bernoulli}(1/2)\) on this class. If the training set contains no sample from \(C_j\), then \(b_j\) remains a fair coin conditional on all observed data. Hence every deterministic predictor has conditional error \(1/2\) on that cell. The absence probability is \((1-p_j)^N\), so the Bayes risk under the product prior is
\[
    \frac12\sum_{j=1}^M p_j(1-p_j)^N.
\]
Yao's minimax principle \citep{Yao1977} implies that this Bayes risk lower-bounds the minimax risk of any randomized example-only learner. For uniform cells, the expression reduces to \(\frac12(1-1/M)^N\). A structure-aware solver that receives \(b\) has zero risk by assumption.

\section{Proof of the unobserved-orbit non-identifiability proposition}
Let \(G\) act on \(\mathcal X\), and suppose the finite training set \(T\) does not intersect an orbit \(O=\{gx:g\in G\}\). Define two equivariant maps identically on every orbit that intersects \(T\). On the unobserved orbit, choose a representative \(x_0\in O\). If the output representation admits two distinct values \(v_1,v_2\) compatible with the stabilizer of \(x_0\), define \(f_i(gx_0)=\rho(g)v_i\) for \(i\in\{1,2\}\). This is well-defined by stabilizer compatibility and is equivariant by construction. The maps agree on \(T\) but disagree on \(O\). The common scalar-invariant case is immediate because all scalar labels are compatible with the stabilizer.

\section{Representative structure-encoded problem classes}
The problem classes discussed in the main paper are summarized in Table~\ref{tab:classes}. The table is not meant as an exhaustive taxonomy; rather, it shows how diverse AI and signal-processing tasks share the same pattern: a compact structural object induces many input-output constraints, and a learner that does not receive that object must infer its consequences from examples.

\small
\begin{longtable}{L{0.20\textwidth}L{0.23\textwidth}L{0.25\textwidth}L{0.22\textwidth}}
\caption{Representative structure-encoded problem classes. The central separation is between using the structural description and relearning its consequences from examples.}\label{tab:classes}\\
\toprule
\textbf{Problem class} & \textbf{Structural object} & \textbf{Structure-aware method} & \textbf{Why a data-only learner pays} \\
\midrule
Known inverse problems & Forward operator, likelihood, noise model & Wiener/Kalman filtering \citep{Wiener1949,Kalman1960}, variational solvers and proximal methods \citep{BoydVandenberghe2004,Rockafellar1976} & Must estimate an operator-induced map; finite-sample excess persists. \\
Symmetry and invariance & Group action, permutation, translation, rotation & Equivariant networks, group convolution, invariant solvers \citep{CohenWelling2016,Zaheer2017,Bronstein2021} & Must cover or infer group orbits; finite data cannot certify unobserved transformations. \\
Sparse and low-dimensional signals & Sensing matrix, sparsity, low rank, dictionary & OMP \citep{TroppGilbert2007}, LASSO \citep{Tibshirani1996}, ISTA/FISTA \citep{Daubechies2004,BeckTeboulle2009}, AMP \citep{DonohoMalekiMontanari2009}, unrolling \citep{GregorLeCun2010,Monga2021} & Must learn a combinatorial atlas of supports, ranks, or active sets. \\
Spectral and signal models & Fourier basis, array manifold, convolutional template & FFT \citep{CooleyTukey1965}, matched filter \citep{Turin1960}, MUSIC/ESPRIT \citep{Schmidt1986,RoyKailath1989}, beamforming \citep{VanTrees2002} & Must rediscover harmonic geometry and phase relationships from examples. \\
Optimization-defined maps & Objective, constraints, KKT conditions & Convex solvers \citep{BoydVandenberghe2004}, dynamic programming \citep{Bellman1957}, proximal algorithms \citep{Rockafellar1976} & Must learn active-set or recurrence structure rather than use optimality conditions. \\
Symbolic and algorithmic reasoning & Grammar, rule base, graph recurrence, program & Parsers \citep{HopcroftUllman1979}, SAT \citep{Cook1971}, MILP \citep{NemhauserWolsey1988}, shortest paths and breadth-first search (BFS) \citep{Dijkstra1959,Cormen2022}, theorem proving \citep{Robinson1965} & Must learn exact discrete transitions and extrapolate to unseen sizes or symbols. \\
Physical and causal systems & Forward model, conservation law, causal graph & PDE solver, simulator, model-predictive control \citep{RawlingsMayne2009} & Must infer governing laws from finite trajectories and may rely on shortcuts that fail under distribution shift or counterfactual intervention \citep{geirhos2020shortcut,scholkopf2021toward}. \\
\bottomrule
\end{longtable}
\normalsize

\section{Continuous-covering variant}
This section expands the continuous-regime discussion following the finite-cell proposition in the main paper. It applies when the relevant cells are indexed by a continuum such as an angle, delay, pose, frequency, or physical parameter. Let the hidden structural regime be parameterized by a compact metric set \(\Theta\) with covering number \(\mathcal N(\Theta,d,\eta)\). Suppose that for any \(\eta\)-ball not represented in the training data, there are two admissible extensions that agree on all observed samples but differ by loss margin \(\gamma\) on a set of probability at least \(p_\eta\). A distribution-free certificate at resolution \(\eta\) requires represented samples in all covering cells. If \(\Theta\) has intrinsic dimension \(q\), then \(\mathcal N(\Theta,d,\eta)\) typically scales as \(\eta^{-q}\). Thus exact continuous-group certification cannot be obtained from finite samples without additional regularity or the structural law.

\section{Sparse atlas interpretation}
This section gives the sparse-recovery specialization of the cell-coverage argument referenced in the main paper. For sparse recovery, the support set has size \(M=\binom{n}{k}\). Conditional on a support \(T\), the noiseless inverse is an affine or linear rule on the measurement subspace generated by \(A_T\). A solver given \(A\) and the sparsity model uses correlations, residual projections, or convex geometry to navigate these cells. A data-only predictor that does not exploit shared structure across supports must cover the support atlas. Under uniform support sampling, the coupon-collector lower bound implies that seeing at least one example from every support with probability \(1-\delta\) requires \(N=\Omega(M\log(M/\delta))\), far larger than the description length of the matrix and solver.

\section{Graph-reachability experiment}
The main paper reports the graph-reachability result at the largest training size. All learned models are trained on balanced source-target reachability problems over 8-node undirected Erd\H{o}s--R\'{e}nyi graphs, with edge probability sampled independently per graph from \(\mathrm{Uniform}(0.05,0.35)\). The source is node 0 and the target is node \(n-1\). Inputs are padded to a common \(24\times24\) adjacency representation. Evaluation uses 2,000 balanced test graphs for each size \(n\in\{8,10,12,16,20,24\}\) and each of five seeds.

The dense fully connected network (FCN) receives the flattened adjacency and active-node mask and has hidden widths 256--256--128 (252,673 trainable parameters). The reported four-layer graph neural network (GNN-4) uses four sum-aggregation message-passing layers with hidden dimension 64, source-target concatenation readout, and no self loops (58,113 trainable parameters). Both use AdamW with learning rate \(10^{-3}\), weight decay \(10^{-4}\), batch size 256, at most 100 epochs, and validation-based early stopping. Breadth-first search (BFS) is evaluated directly at test time \citep{Cormen2022}. Table~\ref{tab:graph} and Figure~\ref{fig:graph-supp} report mean performance and 95\% Student-\(t\) confidence intervals over the five seeds.

\begin{table}[h]
\centering
\small
\begin{tabular}{lccc}
\toprule
Method & Accuracy, \(n=8\) & Accuracy, \(n=24\) & Loss 8--24 (pp) \\
\midrule
BFS & \(100.00\pm0.00\) & \(100.00\pm0.00\) & \(0.00\pm0.00\) \\
FCN & \(97.79\pm0.81\) & \(73.05\pm3.61\) & \(24.74\pm3.76\) \\
GNN-4 & \(99.97\pm0.06\) & \(95.48\pm0.80\) & \(4.49\pm0.79\) \\
GNN-8 & \(99.98\pm0.03\) & \(96.03\pm1.80\) & \(3.95\pm1.78\) \\
\bottomrule
\end{tabular}
\caption{Graph reachability at \(N=16{,}384\). Entries are means \(\pm\) 95\% confidence intervals over five seeds.}
\label{tab:graph}
\end{table}

\begin{figure}[H]
    \centering
    \includegraphics[width=\textwidth]{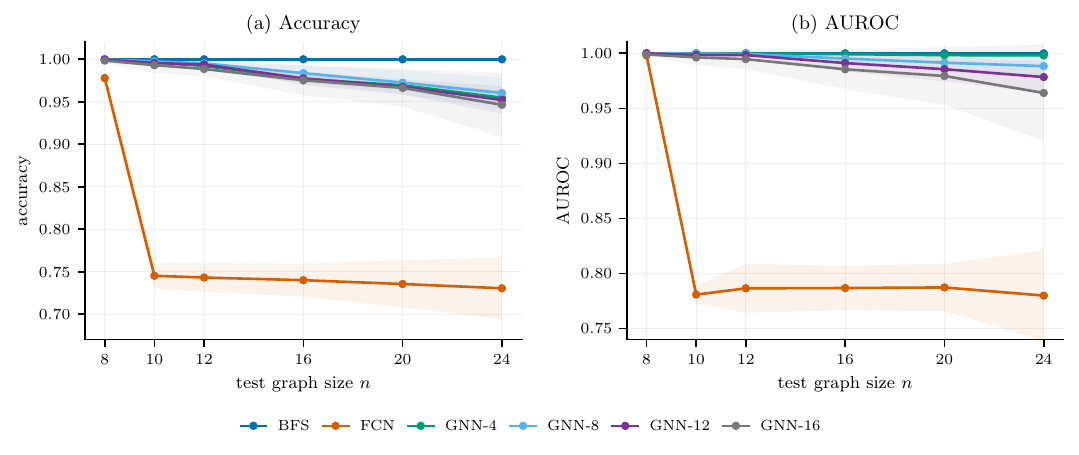}
    \caption{Graph-size extrapolation for all tested message-passing depths. All learned models are trained on 8-node graphs with \(N=16{,}384\). Curves show means and 95\% confidence bands over five seeds.}
    \label{fig:graph-supp}
\end{figure}

The graph result should be understood as extrapolation within a common 24-node padded representation envelope, not as arbitrary-size algorithm execution. The generator also changes expected degree as \(n\) grows because the edge-probability range is fixed. These choices make the experiment a controlled demonstration of relational inductive bias rather than a universal benchmark of graph reasoning.

\section{Additional numerical controls}
The matched-filter and permutation-equivariance experiments provide additional examples of the continuous-orbit and finite-augmentation mechanisms.

\begin{figure}[t]
    \centering
    \begin{subfigure}[t]{0.48\textwidth}
        \centering
        \includegraphics[width=\linewidth]{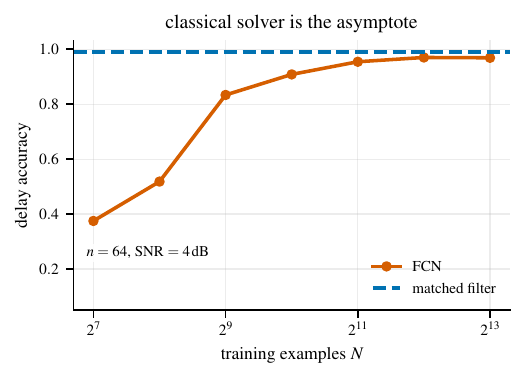}
        \caption{Matched-filter delay estimation.}
    \end{subfigure}\hfill
    \begin{subfigure}[t]{0.48\textwidth}
        \centering
        \includegraphics[width=\linewidth]{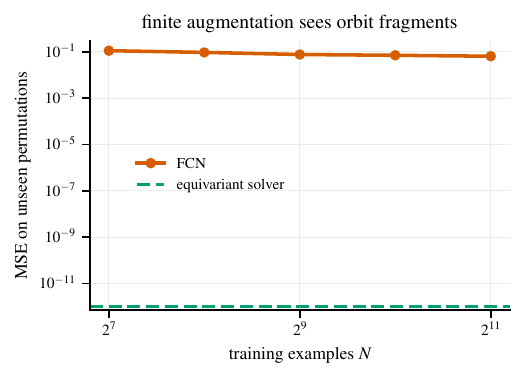}
        \caption{Permutation-equivariant regression.}
    \end{subfigure}
    \caption{Additional controlled examples. The matched filter knows the template orbit; the equivariant computation knows the permutation law. In both cases the FCN must recover the corresponding implications from finite samples.}
    \label{fig:additional-controls}
\end{figure}

\section{Reproducibility}
A general-purpose large language model was used to assist with drafting, organization, code generation, and editing of this manuscript.  The arXiv submission contains the manuscript sources and figures. The accompanying code-and-data archive includes experiment scripts, result files, requirements, and figure-generation code. Simulation plots can be regenerated from the included results using Matplotlib; conceptual diagrams are separate schematics. The dense-network baselines are structure-agnostic.  Any neural architecture that hard-codes the relevant operator, group action, graph topology, recurrence, or optimizer is counted as structure-aware under the definitions in the main paper.

\clearpage
\bibliography{references}

\end{document}